\documentclass{article}

\usepackage{float}

\usepackage{arxiv}

\usepackage[american]{babel}
\usepackage{natbib} % has a nice set of citation styles and commands
\usepackage{mathtools} % amsmath with fixes and additions
\usepackage{booktabs} % commands to create good-looking tables
\usepackage{tikz} % nice language for creating drawings and diagrams

\DeclareMathOperator*{\argmin}{arg\,min}
\DeclareMathOperator*{\argmax}{arg\,max}

\usepackage{array}
\usepackage{multirow}
\usepackage{amsthm}
\usepackage{amssymb}

\newtheorem{definition}{Definition}
\newtheorem{lemma}{Lemma}
\newtheorem{proposition}{Proposition}

\newtheorem{remark}{Remark}
\newtheorem{example}{Example}

\usepackage{caption}
\usepackage{subcaption} 
\usepackage{enumitem}
\usepackage{hyperref}
\usetikzlibrary{arrows}
\usetikzlibrary{positioning}
\usepackage[ruled,vlined,linesnumbered,nosemicolon]{algorithm2e}

\SetKwFunction{FnDA}{DA}
\SetKwProg{Fn}{Function}{:}{}

\title{Upper entropy for 2-monotone lower probabilities}

\author{
Tuan-Anh Vu \\
	UMR CNRS 7253 Heudiasyc\\ Université de Technologie de Compiègne\\
    France\\
    \texttt{tuan-anh.vu@hds.utc.fr} 
	\And
	Sébastien Destercke \\
	UMR CNRS 7253 Heudiasyc\\ Université de Technologie de Compiègne\\
    France\\
    \texttt{sebastien.destercke@hds.utc.fr} 
    \And
	Frédéric Pichon \\
	Laboratoire de Génie Informatique et d’Automatique de l’Artois (LGI2A)\\
    Université d'Artois, UR 3926, France\\
	\texttt{frederic.pichon@univ-artois.fr} \\
}

\begin{document}
\maketitle

\begin{abstract}
Uncertainty quantification is a key aspect in many tasks such as model selection/regularization, or quantifying prediction uncertainties to perform active learning or OOD detection. Within credal approaches that consider modeling uncertainty as probability sets, upper entropy plays a central role as an uncertainty measure. This paper is devoted to the computational aspect of upper entropies, providing an exhaustive algorithmic and complexity analysis of the problem. In particular, we show that the problem has a strongly polynomial solution, and propose many significant improvements over past algorithms proposed for 2-monotone lower probabilities and their specific cases. 
\end{abstract}

\section{Introduction}\label{sec:intro}

Imprecise probabilistic or \emph{credal}~\citep{levi1980enterprise} approaches offer rich representations of uncertainty, usually coming in the form of convex sets of probabilities or of equivalent representations such as lower envelope of expectations~\citep{troffaes2014lower}, that we will call here \emph{credal sets}. They have been used in many facets of AI, such as symbolic reasoning~\citep{hansen2000probabilistic}, graphical models~\citep{cozman2000credal} or machine learning~\citep{caprio2024credal}, with this latter trend gaining traction, as show many recent papers on the topic~\citep{singh2024domain,wang2024credal,nguyen2025credal,lohr2025credal}. 
Uncertainty quantification, by which we mean the fact of quantifying some aspect of an uncertainty representation by a value, is another important aspect of reasoning under uncertainty, and is a cornerstone of, e.g., Klir's theory~\citep{klir2006uncertainty}. Such uncertainty quantifications can serve, e.g., to select a peculiar model among many possible ones~\citep{dubey2018maximum}, as a regularization tool when used as a penalty term~\citep{grandvalet2006entropy}, or as a way to quantify a predictive uncertainty. 

Within credal approaches, upper entropy plays a particular role, as it can be shown to satisfy quite a number of axioms~\citep[Sec. 6.7.]{klir2006uncertainty} when considered  as measure of total uncertainty. It can also be associated to a robust, minimax version of using the logarithmic scoring rule: if your uncertainty is represented by a credal set $\mathcal{P}$, and if $\overline{\mathbb{E}}_\mathcal{P}[-\ln(Q)]$ is the maximum loss of choosing $Q$ under knowledge $\mathcal{P}$ with a logarithmic scoring rule, then the distribution $Q$ minimizing $\overline{\mathbb{E}}_\mathcal{P}[-\ln(Q)]$ is the one of maximal entropy~\cite[Sec. 5.12.]{walley1991statistical}. It is also commonly used in various learning settings, such as the learning of decision trees~\citep{abellan2005upper} or as a means to quantify total uncertainties of predictions~\citep{hullermeier2021aleatoric,javanmardi2024conformalized}. 

Given this importance of upper entropy, it is surprising that the computational aspects of calculating it for large families of credal sets has not been much explored. \cite{abellan2006algorithm} provide an algorithm for credal sets induced by 2-monotone capacities, yet only conclude that it is a ``difficult problem''. This algorithm is derived from the MRW algorithm~\citep{meyerowitz1994calculating,harmanec1996computation} that is dedicated to belief functions and is claimed to be of exponential complexity in the size of the considered space by~\cite{huynh2009notes}. Recent improvements of the MRW do not discuss complexity either~\citep{e25060867}. A more efficient algorithm, that is quadratic in the size of the space, exists for the specific case of probability intervals~\citep{abellan2003maximum}. 

The goal of this paper is to revisit the problem of computing upper entropies for credal sets $\mathcal{P}$ induced by 2-monotone (a.k.a. supermodular) capacities, that include many special cases such as belief functions, probability intervals, $\epsilon$-contamination models or possibility distributions. Our main contributions to this problem are the following:
\begin{itemize}
\item We demonstrate that by using supermodular optimization, the exact algorithm proposed by \cite{abellan2006algorithm} can be made polynomial by making a number of calls to supermodular function
maximization (SFM) that is quadratic in the space size. This insight also allows us to propose a new algorithm whose number of calls to SFM is linear in the space size. These results both show that the problem is not as hard as previously claimed, and improve significantly upon existing  methods. This is elaborated in Section~\ref{sec:Stronly Polynomiality}.
\item We provide improved algorithms for the specific cases of belief functions, possibility distributions and probability intervals, that are commonly considered in applications (see, e.g.,~\citep{caprio2025credal,chau2025integral}). These are detailed in Section~\ref{sec: special case}.
\item We propose in Section~\ref{sec:FW_algo} an approximation algorithm using the Frank--Wolfe  algorithm~\citep{Bachsubmodular13} that frames the true results between a lower and a upper bounds, thereby allowing us to control the approximation error. 
\end{itemize}
Finally, we propose in Section~\ref{sec:expe} some first experiments that compare our algorithms but also prove that our improvements make the associated computations scalable. Before that, let us start with some necessary preliminaries.

\section{Preliminaries}
\subsection{2-monotone lower probability}

Throughout this paper, we consider a finite space $\Omega=\{1,2,\ldots,n\}$. A set function $\nu: 2^\Omega \to \mathbb R$ with $\nu(\emptyset)= 0$ is called \emph{submodular} if 
\begin{equation}\label{eq:submodular_defintion}
    \nu(A) + \nu(B) \geq \nu(A\cup B) + \nu(A\cap B)~\forall A,B \subseteq \Omega. 
\end{equation}
Similarly, $\nu$ is called \emph{supermodular} if $-\nu$ is submodular, i.e., the inequality in (\ref{eq:submodular_defintion}) is reversed. 
As listing $2^n$ values is intractable, we assume that $\nu$ is accessed through an oracle that returns $\nu(A)$ for any query $A \subseteq \Omega$ \citep[Chapter 10]{grotschel1993geometric} and let $\mathrm{EO}$ be the time required for one query. The \emph{base polyhedron} $B(\nu)\subseteq \mathbb R^n$ of $\nu$ is defined as
\begin{equation}\label{eq:base_polyhedron}
    B(\nu)=\{x: \sum_{i\in A}x_i \leq \nu(A)~\forall A\subseteq \Omega,~ \sum_{i=1}^nx_i = \nu(\Omega)\}.
\end{equation}

Submodular (supermodular) functions appear in many domains, including the theory of imprecise probabilities \citep{walley1991statistical}. This uncertainty theory generalizes probability theory, that uses \emph{additive measures}, notably by considering probability sets that can be induced by \emph{lower probabilities}, that are \emph{non-additive measures}. A set function $\mu : 2^\Omega \to [0, 1]$ is called a lower probability if $\mu(\Omega) = 1,~\mu(\emptyset) = 0~\text{ and }\mu(A) \leq \mu(B)$ if $A\subseteq B$. The term ``lower probability'' comes from the interpretation that the true probability on $\Omega$ is unknown and $\mu$ encodes the available (limited) information in the form of lower bounds on event probabilities. Hence, the so-called \emph{credal set} consisting of probabilities on $\Omega$ that are compatible with $\mu$ is $\mathcal{P}(\mu):= \{p \in \Delta_n: p(A) \geq \mu(A)~\forall A \subseteq \Omega\}.
$

Many important lower probabilities are supermodular, and have a specific name in imprecise probabilities.
\begin{definition}
   A lower probability $\mu$ is called \emph{$2$-monotone} if it is supermodular.
\end{definition}
\begin{remark}\label{remark:credalset=base poly}
    For a lower probability $\mu$, its dual, called an \emph{upper probability} $\bar{\mu}$, is given by $\bar{\mu}(A)=1-\mu(\Omega\setminus A)~\forall A$. It can be seen that $ \mathcal{P}(\mu) = \{p: p(A) \leq \bar{\mu}(A)~\forall A\subseteq \Omega\}.$ Moreover, $2$-monotonicity of $\mu$ implies that $\bar{\mu}$ is submodular, so $\mathcal{P}(\mu)$ is exactly the base polyhedron $B(\bar{\mu})$ of $\bar{\mu}$ (see \eqref{eq:base_polyhedron}).
\end{remark}
Below, we highlight some notable special cases of $2$-monotone lower probabilities.

\paragraph{Belief Function}
A mass function on $\Omega$ is a mapping $m: 2^\Omega \to [0,1]$ satisfying $\sum_{A\subseteq \Omega}m(A)=1$ and $m(\emptyset)=0$ \citep{shafer1976mathematical}. A subset $A$ is called a \emph{focal set} of $m$ if $m(A)>0$. $m$ induces a special lower probability called \emph{belief function} defined as $\mathrm{Bel}(B):= \sum_{A\subseteq B}m(A)$. The dual of $\mathrm{Bel}$, the plausibility function, is defined as $\mathrm{Pl}(B):= \sum_{A\cap B\ne \emptyset}m(A)$. In this case, $\mathrm{EO}$ is polynomial in $|\{A:m(A) > 0\}|$, so is tractable if we do consider a mass function positive on a limited amount of subsets: this is for instance the case with k-additive belief functions~\citep{grabisch1997k} or imprecise cumulative distributions~\citep{montes2017extreme}. 

\paragraph{Possibility Distribution}
A possibility distribution $\pi$ on $\Omega$ is a vector in $[0,1]^n$ where each component $\pi_i$ represents the degree of possibility of element $i$.  $\pi$ induces a lower probability on $\Omega$, called the necessity measure, defined as
$\mathrm{N}(A):=1 -\max_{i \in \Omega \setminus A}\pi_i ~\forall A$. The dual of $\mathrm{N}$, the possibility measure, is $\Pi(A):= \max_{i \in  A}\pi_i ~\forall A $. Again, this means $\mathrm{EO}$ is polynomial. 

\paragraph{Probability Intervals} In this case, the credal set is given directly as
$ \mathcal{P}= \{p \in \Delta_n: l_i \leq p_i\leq u_i~\forall i\},$
where each intervals $[l_i, u_i] \subseteq [0,1]~\forall i \in \{1, \ldots, n\}$. Moreover, we require that  $\sum_{i=1}^nl_i \leq 1 \leq \sum_{i=1}^nu_i$ so that $\mathcal{P}\ne \emptyset$. \citet{de1994probability} showed that the set function $\mu$ where $\mu(A):=\min_{p\in \mathcal{P}}p(A)~\forall A$ is a 2-monotone lower probability. Note that \citet{de1994probability} provide formulas such that $\mathrm{EO}$ is polynomial in $|\Omega|$.

\subsection{Upper Entropy}
% Let $\mu$ be a 2-monotone lower probability whose credal set is $\mathcal{P}(\mu)$.
The upper entropy $\mathrm{UE}(\mu)$ of a lower probability $\mu$ is defined as
\begin{equation}\label{prob:upper_entropy}
    \mathrm{UE}(\mu):= \max_{p \in \mathcal{P}(\mu)}-\sum_{i=1}^np_i \log p_i.
\end{equation}

Algorithm \ref{algo:AM} presents the \citet{abellan2006algorithm} procedure for computing $\mathrm{UE}(\mu)$ when $\mu$ is 2-monotone. Setting $\mu = \mathrm{Bel}$ yields the MRW algorithm \citep{meyerowitz1994calculating,harmanec1996computation} for belief functions. However and as said in the introduction, none of these papers proceed through a detailed complexity analysis of these algorithms, merely concluding that the problem is ``difficult'', mostly due to Line 1 in Algorithm~\ref{algo:AM}. 

\begin{algorithm}
\caption{Abellán-Moral's algorithm}
\label{algo:AM}
\KwIn{2-monotone $\mu$ on $\Omega =\{1, \ldots,n\}$.}

Find $A$ such that $\frac{\mu(A)}{|A|}$ is maximum. If $A$ is not unique, the largest such set is selected \label{Line:1}

For each $i \in A$, $p_i = \frac{\mu(A)}{|A|}$

For each $B \subseteq \Omega \setminus A$, $\mu(B) = \mu(B \cup A) - \mu(A)$ \label{Line:3}

$\Omega = \Omega \setminus A$ \label{Line:4}

\lIf{$\Omega \neq \emptyset$ and $\mu(\Omega) > 0$}{
    go to 1 \label{Line:5}
}

\lIf{$\mu(\Omega) = 0$ and $\Omega \neq \emptyset$}{
    $p_i = 0$ for all $i \in \Omega$
}

\Return{ $-\sum_{i=1}^np_i \log p_i$\;}

\end{algorithm}

\section{Strong Polynomiality of computing Upper Entropy}\label{sec:Stronly Polynomiality}

We call an algorithm whose input is a supermodular (submodular) function \emph{strongly polynomial} if it performs a number of oracle queries and arithmetic operations \footnote{Arithmetic operations include addition, subtraction, multiplication, division, comparison operations, etc.} bounded by $\operatorname{poly}(n)$. A notable example is supermodular function maximization (SFM) (equivalently, submodular function minimization), which seeks a subset maximizing (resp. minimizing) a supermodular (resp. submodular) set function; several algorithms are known, e.g., Orlin’s algorithm \citet{orlin2009faster} runs in time $O(n^5\mathrm{EO} + n^6)$. Upon closer inspection, Algorithm \ref{algo:AM} can be implemented so that it is strongly polynomial. The key observation is that although the optimization problem in Line \ref{Line:1} is not an SFM, it can be transformed into a series of SFM. 

\begin{lemma}\label{Lemma:1}
    Finding $A\in \argmax_{\emptyset \neq B \subseteq \Omega} \frac{\mu(B)}{|B|}$ and $|A|$ is maximum amounts to solving $O(n)$ SFM.
\end{lemma}

The key of the proof, given in Appendix \ref{appendix: proof of lemma 1}, is that the function $\mu(A) - \lambda |A|$ is supermodular if $\mu$ is supermodular, and that finding the argmax within Lemma~\ref{Lemma:1} is equivalent to successively finding the set $A$ maximizing $\mu(A) - \lambda |A|$ (a SFM procedure, known to be polynomial), for $O(n)$ suitable values of $\lambda$.

\begin{proposition}\label{prop:strong_poly}
    Algorithm \ref{algo:AM} boils down to solving $O(n^2)$ SFM, and thus is strongly polynomial.
\end{proposition}
\begin{proof}
Line \ref{Line:1} of Algorithm \ref{algo:AM} requires solving $O(n)$ SFM (Lemma \ref{Lemma:1}). Moreover, the update in Lines \ref{Line:3}--\ref{Line:4} still results in a supermodular function on $\Omega\setminus A$, e.g., \citep[Property 1]{abellan2006algorithm}, whose evaluation oracle is given by Line \ref{Line:3}. As $\Omega$ shrinks by at least one element per iteration (Line \ref{Line:4}), Algorithm \ref{algo:AM} terminates after at most $n$ passes through the loop in Lines \ref{Line:1}--\ref{Line:5}. Hence, it requires solving $O(n^2)$ SFM.
\end{proof}

As it turns out, we can compute $\mathrm{UE}(\mu)$ with improved complexity by using a decomposition algorithm \citep{NaganoK13, fujishige2005submodular}. Let $\nu$ be a submodular function and $b$ be a positive vector in $\mathbb {R}^n$, \citet[Sec. 5]{NaganoK13} shows that solving
\begin{equation}\label{prob:Nagano}
   \min_{x\in B(\nu)}\sum_{i=1}^n (x_i\log\frac{x_i}{b_i} + b_i - x_i)
\end{equation}
reduces to finding a chain $\emptyset=S_0\subset\cdots\subset S_l=\Omega$ and breakpoints $0=\alpha_0<\alpha_1<\cdots<\alpha_l<\alpha_{l+1}=+\infty$ such that, for any $\alpha\in[\alpha_j,\alpha_{j+1})$, $S_j$ is the unique minimizer of maximum cardinality of $\min_{A\subseteq\Omega} \nu(A)-\alpha\sum_{i\in A}b_i$. The minimizer of (\ref{prob:Nagano}) is, for each $i \in S_{j} \setminus S_{j-1} ~(j \in {1,\ldots, l})$
\begin{equation}\label{eq:optimal_nagano}
    x_i = \frac{\nu(S_{j})-\nu(S_{j-1})}{\sum_{i\in S_{j}\setminus S_{j-1}}b_i}b_i.
\end{equation}
Algorithm \ref{algo:decomposition} aims to compute the chain $\{S_j\}$. Given two sets $S_j \subset S_k$ in the chain (initially $S_j= \emptyset$, $S_k= \Omega$), it recursively checks if $S_j$ and $S_k$ are consecutive, and if not, it finds a new chain set strictly between them and repeats the process until the whole chain is revealed. In total, \citet{NaganoK13} shows that Algorithm \ref{algo:decomposition} requires solving $O(n)$ SFM (Line \ref{Line:SFM_solving}).

\begin{algorithm}
\caption{Decomposition algorithm}
\label{algo:decomposition}
\KwIn{submodular $\nu$ on $\Omega =\{1, \ldots,n\}$; $b\in \mathbb {R}^n $.}
\Fn{\FnDA{$S, S'$}}{
        $\alpha = \dfrac{\nu(S')-\nu(S)}{\sum_{i\in S'\setminus S}b_i}$\;
        Take $S'' \in \argmin_{A\subseteq\Omega} \nu(A)-\alpha\sum_{i\in A}b_i$ s.t $|S''|$ is largest\label{Line:SFM_solving}\;
        
        % Using \lIf for single-line If
        \lIf{$S''=S'$}{\Return{$\{S,S'\}$}}
        
        % Using \lElse for single-line Else
        \lElse{\Return{\FnDA{$S,S''$} $\cup$ \FnDA{$S'',S'$}}}
    }

\BlankLine
\Return $\mathcal{S}^* = \FnDA(\emptyset,\Omega)$\;
\end{algorithm}

\begin{proposition}
    Computing $\mathrm{UE}(\mu)$ amounts to running Algorithm \ref{algo:decomposition} with $\bar{\mu}$ (the dual of $\mu$) and $b=\mathbf 1$, which requires $O(n)$ SFM.
\end{proposition}
\begin{proof}
    With the provided input, Problem \eqref{prob:Nagano} becomes $\min_{x\in B(\nu)}\left(\sum_{i=1}^n x_i\log x_i\right) + n - 1$. It has the same optimal solution as \eqref{prob:upper_entropy} because of Remark \ref{remark:credalset=base poly}. Algorithm \ref{algo:decomposition} outputs the chain $\{S_j\}$ and the optimal $p$ of \eqref{prob:upper_entropy} is obtained via (\ref{eq:optimal_nagano}).
\end{proof}

\begin{remark}
    There is a tight connection between Algorithm \ref{algo:AM} and Algorithm \ref{algo:decomposition}: the sets $A_1, A_2,\ldots$ successively computed in Line \ref{Line:1} of the former are exactly the sets $S_l\setminus S_{l-1}, S_{l-1}\setminus S_{l-2},\ldots $ from the chain $\emptyset=S_0\subset\cdots\subset S_l=\Omega$ returned by the latter with the input $\bar{\mu}$ and $\mathbf{1}$. 
\end{remark}
\begin{example}
Consider a 2-monotone $\mu$ and its dual $\bar{\mu}$ on $\Omega = \{1,2,3\}$. Their values are given below (excluding the obvious values of $\emptyset$ and $\Omega$). 
Algorithm \ref{algo:AM}, with input $\mu$, outputs the sets $\{2,3\}$ and $\{1\}$ in this order at Line \ref{Line:1}. Hence, the optimal probability is $p= (0.2, 0.4, 0.4)$.
$$
\begin{array}{c|ccccccc}
 A & \{1\} & \{2\} & \{3\} & \{1,2\} & \{1,3\} & \{2,3\} \\
 \hline
 \mu(A) & 0.1 & 0.4 & 0.4 & 0.5 & 0.5 & 0.8\\
 \bar{\mu}(A) & 0.2 & 0.5 & 0.5 & 0.6 & 0.6 & 0.9
\end{array}$$
Let us illustrate Algorithm \ref{algo:decomposition} with $\bar{\mu}$ and $b=\mathbf 1$.
First, we run $\FnDA(\emptyset,\Omega)$, for which the associated $\alpha= \frac{\bar{\mu}(\Omega)-\bar{\mu}(\emptyset)}{|\Omega|}=\frac{1}{3}$. Second, we solve $\min_{A\subseteq\Omega}\bar{\mu}(A)-\frac{1}{3}|A|$ and obtain $A=
\{1\}$ the maximal minimizer. Third, we would proceed recursively by executing $\FnDA(\emptyset,\{1\})$ and $\FnDA(\{1\}, \Omega)$; but since there is no set strictly between $\emptyset$ and $\{1\}$, we only need to run $\FnDA(\{1\}, \Omega)$, for which $\alpha=\frac{\bar{\mu}(\Omega)-\bar{\mu}(\{1\})}{|\Omega\setminus \{1\}|}=0.4$. Finally, solving $\min_{A\subseteq\Omega}\bar{\mu}(A)-0.4|A|$, we find that $\Omega$ is the maximal minimizer, implying there is no set in the chain strictly between $\{1\}$ and $\Omega$, so we stop. Hence, Algorithm \ref{algo:decomposition} outputs the chain $\emptyset \subset \{1\} \subset \Omega$. Applying (\ref{eq:optimal_nagano}), we also get $p= (0.2, 0.4, 0.4)$.

\end{example}

\section{Computing upper entropy in special cases}\label{sec: special case}
\subsection{Belief function}

Let $m$ be a mass function on $\Omega$ with $\mathrm{Pl}$ its plausibility function  and $\mathcal{F}$ its set of focal sets. We consider a directed bipartite graph with partitions $L$ and $R$ in which each \emph{left vertex} in $L$ (resp. \emph{right vertex} in $R$) corresponds to an element of $\Omega$ (resp. focal set of $m$), and an edge from $j$ to $F_i$ iff $j \in F_i$. From this graph, we construct a flow network $G$ by adding a source $s$ and a sink $t$, together with edges $(s,j)$ for all elements $j \in \Omega$ and $(F_i,t)$ for all focal sets $F_i$ of $m$. Define a capacity function $c$ on edges as $c(s,j)=\alpha$, $c(j,F_i)=+\infty$, and $c(F_i,t)=m(F_i)$.
A cut $(S, T)$ of $G$ is simply a partition of its vertices into a \emph{source set} $S$ (where $s \in S$) and a \emph{sink set} $T = V \setminus S$ (where $t \in T$)  \citep[Chapter 26]{cormen2009introduction}. The capacity of $(S,T)$ is the sum of capacities of edges leaving $S$, i.e., $c(S,T)= \sum_{u\in S}\sum_{v \in T}c(u,v)$. As the cut $(S,T)$ where $T= \{t\}$ has capacity 1, the minimum cut capacity is at most $1$.

In this case, Line \ref{Line:SFM_solving} of Algorithm \ref{algo:decomposition} boils down to minimum-cut/maximum-flow computations and is thus much faster than general SFM.
\begin{proposition}\label{prop:min_cut_belief}
     Solving $\min_{A \subseteq \Omega}\mathrm{Pl}(A)-\alpha|A|$ with $\alpha>0$ is equivalent to finding a minimum cut $(S,T)$ in $G$.
\end{proposition}
\begin{proof}
Consider a minimum cut $X=(S,T)$ where $S = {s} \cup L^\prime \cup R^\prime$, with $L^\prime$ a set of left vertices (elements of $\Omega$) and $R^\prime$ a set of right vertices (focal sets of $m$). If an element $i\in L^\prime$, then $R^\prime$ contain all focal sets $F$ such that $i \in F$, otherwise there is an edge with capacity $+\infty$ leaving $S$. Furthermore, if there is a focal set $F\in R^\prime$, then $L^\prime\cap F \ne \emptyset$, otherwise the cut formed by removing $F$ from $R^\prime$ has capacity of $c(X)-m(F) < c(X)$. Therefore, $X$ is of the form with
$ S = \{s\} \cup A \cup \{F\in \mathcal{F}: F \cap A \ne \emptyset\},  \text{ for a subset $A\subseteq \Omega$}.$ By the construction of $G$,
\begin{align*}
     &c(X)= \alpha|\Omega \setminus A| + \sum_{F\cap A\ne \emptyset}m(F)\\ &= \sum_{F\cap A \ne \emptyset}m(F) - \alpha|A| + n\alpha= \mathrm{Pl}(A) -\alpha|A| +n\alpha,
\end{align*}
so $\min_{A \subseteq \Omega}\mathrm{Pl}(A)-\alpha|A|= c^* - n\alpha$ with $c^*$ is the minimum cut capacity. The proposition holds as $n\alpha$ is constant. 
\end{proof}
\begin{example}\label{example:flow}
    Consider a mass $m$ on $\Omega := \{1,2,3,4\}$ whose focal sets are $F_1 =\{1,2\}, F_2 =\{2,3\}$ and $F_3 = \{3,4\}$. The flow network from Proposition \ref{prop:min_cut_belief} is shown in Figure \ref{fig:flow-network}, which we use to illustrate its proof. Namely, the cut with $S = \{s\} \cup \{1\} \cup \{F_1\}$ is not minimum because edge $(1,F_2)$ has capacity $+\infty$. The cut with $S = \{s\} \cup \{1,2\} \cup \{F_1, F_2, F_3\}$ is not minimum because removing $F_3$ from $S$ results in a cut with strictly smaller capacity. 
\end{example}

We still need to select $A^*\in \argmin_{A\subseteq\Omega}\mathrm{Pl}(A) - \alpha |A|$ with maximum cardinality. It is well-known that for a minimum cut $(S^*, T^*)$ such that $|S^*|$ is maximum, $S^*$ is the union of all minimum cut source sets \citep{PicardQueyranne1980MinCuts}. It follows that $A^*$ consists of exactly the left vertices in $S^*$, found as follows. By the max-flow min-cut theorem \citep[Theorem 26.6]{cormen2009introduction}, we first compute a maximum flow of $G$ and construct the associated residual graph. Then $A^*$ is the set of left vertices that cannot reach $t$ in this residual graph \citep{PicardQueyranne1980MinCuts}. Hence, we obtain: 
\begin{proposition}
    For belief functions, computing upper entropy via Algorithm \ref{algo:decomposition} requires $O(n)$ maximum-flow computations and $O(n)$ depth-first searches.
\end{proposition}
% Hence, Algorithm \ref{algo:decomposition} requires $O(n)$ maximum-flow computations and $O(n)$ depth-first searches. 
As the number of edges of $G$ and its residual graphs are bounded by $O(n|\mathcal{F}|)$, the running time of Algorithm \ref{algo:decomposition} is $\mathrm{poly}(n,|\mathcal{F}|)$ with the exact complexity depending on the chosen maximum-flow algorithm.

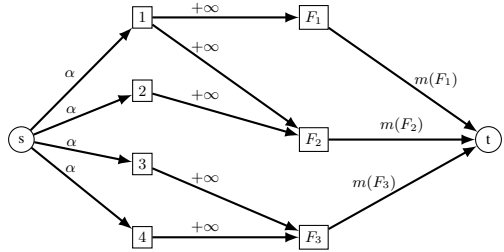
\begin{figure}[H]
    \centering
        \begin{tikzpicture}[
            scale=0.65,
            transform shape,
            edge/.style={draw, ->, thick, >=latex},
            every node/.style={draw,font=\small}
        ]
        % Nodes
        \node (s) [circle] {s};

        \node (o1) [rectangle,  right= of s, xshift=1cm, yshift=2.5cm] {$1$};
        \node (o2) [rectangle, right= of s, xshift=1cm, yshift=1cm] {$2$};
        \node (o3) [rectangle, right= of s, xshift=1cm, yshift=-0.5cm] {$3$};
        % \node (o4) [rectangle, right= of s, xshift=0.5cm, yshift=-1cm] {$4$};
        \node (o4) [rectangle, right= of s, xshift=1cm, yshift=-2cm] {$4$};

        \node (F1) [rectangle, right=of o1, xshift=2cm, yshift=0cm] {$F_1$};
        \node (F2) [rectangle, right=of o3, xshift=2cm,  yshift=0.5cm] {$F_2$};
        \node (F3) [rectangle, right=of o4, xshift=2cm, yshift=0cm] {$F_3$};

        \node (t) [circle, right=of F2, xshift=2cm] {t};

        \draw[edge] (F1) -- (t) node[midway, above=0.8cm, right=0.1cm, draw=none] {$m(F_1)$};
        \draw[edge] (F2) -- (t) node[midway, above=0cm, draw=none] {$m(F_2)$};
        \draw[edge] (F3) -- (t) node[midway, above=0.1cm, left, draw=none] {$m(F_3)$};
        %intermediate edges

        \draw[edge] (o1) -- (F1)  node[midway, above=0.2cm, left, draw=none] {$+\infty$};
        \draw[edge] (o1) -- (F2) node[midway, above=0.6cm, left, draw=none] {$+\infty$};
        \draw[edge] (o2) -- (F2) node[midway, above=0.4cm, left, draw=none] {$+\infty$};
        \draw[edge] (o3) -- (F3) node[midway, above=0.45cm, left, draw=none] {$+\infty$};

        \draw[edge] (o4) -- (F3) node[midway, above=0.2cm, left, draw=none] {$+\infty$};
        % \draw[edge] (o5) -- (F3) node[midway, below=0.2cm, left, draw=none] {$+\infty$};

        \draw[edge] (s) -- (o1) node[midway, above=0.1cm, left, draw=none] {$\alpha$};
        \draw[edge] (s) -- (o2) node[midway, above=0.1cm, left, draw=none] {$\alpha$};
        \draw[edge] (s) -- (o3) node[midway, above=0.2cm, left, draw=none] {$\alpha$};
        % \draw[edge] (s) -- (o4) node[midway, above=0.2cm, left, draw=none] {$\alpha$};
        \draw[edge] (s) -- (o4) node[midway, above=0.4cm, left, draw=none] {$\alpha$};
        \end{tikzpicture}
        \caption{Flow network in Example \ref{example:flow}.}
        \label{fig:flow-network}
\end{figure}

\subsection{Possibility distribution}
\citet[Algorithm 6.2]{klir2006uncertainty} has particularized Algorithm \ref{algo:AM} for the case of a possibility distribution $\pi$ under the standard convention that $1= \pi_1 \geq \pi_2 \geq \ldots \geq \pi_n >0$. Unsurprisingly, this particularization inherits the quadratic complexity (see Proposition~\ref{prop:strong_poly}), yielding an $O(n^2)$ running time.
We will, however, design an algorithm achieving $O(n)$ complexity.

Recall in Section \ref{sec:Stronly Polynomiality} that our goal is to compute the chain $\emptyset=S_0\subset\cdots\subset S_l=\Omega$ where $S_j$ is the unique maximal minimizer of $\min_{A\subseteq \Omega}\Pi(A)-\alpha |A|$ $~\forall \alpha\in[\alpha_j,\alpha_{j+1})$. To this end, we employ a procedure different from Algorithm \ref{algo:decomposition} that specifically exploits the structure of $\pi$.
Due to the ordering convention, $\Pi(A)= \pi_{j}$ where $j$ is the minimum element of $A$, and thus $\{{j},  \ldots, n\}$ is the largest set having the same possibility as $A$. Therefore,$~\forall \alpha \geq 0$,

\begin{equation}\label{eq:possibility_form}
    \min_{A\subseteq \Omega}\Pi(A)-\alpha |A| = \min_{j \in \{1, \ldots, n, n+1\}} \pi_j - \alpha (n-j+1),
\end{equation}
with $\pi_{n+1} := 0$ to handle the case when $\emptyset$ minimizes the left-hand side at $\alpha = 0$. Let $g_j(\alpha):= \pi_j - \alpha (n-j+1)$ and $g:= \min_{j \in \{1, \ldots, n+1\}}g_j$. To find the desired chain, we compute all breakpoints in $[0, +\infty)$ of $g$ and select the smallest $j$ such that $g_j = g$ at each breakpoint. This ensures that the associated minimizer of the left-hand side of \eqref{eq:possibility_form} has maximum size, which is $\{j, \ldots, n\}$.

Our task can be done via a classic problem in computational geometry. More precisely,
in this field, $g$ is called the \emph{lower envelope} of lines $g_j$. Under the duality that maps a line $\ell :y=m\alpha+b$ to its dual point $\ell^*:= (m, -b)$, $g$ corresponds one-to-one with the \emph{upper convex hull} \footnote{The upper convex hull is the chain of edges on the convex hull’s upper boundary.} of the set of dual points $\{g^*_j: 1\leq j \leq  n+1 \}$ where $g^*_j= (j-n-1, -\pi_j)$ \citep[Chapter 11.4]{de2008computational}. 

We employ this correspondence to our setting to solve $\min_{A\subseteq \Omega}\Pi(A)-\alpha |A|~\forall \alpha$, described in Algorithm \ref{algo:UE_pos}. Because of this correspondence, we simply traverse the upper hull’s vertices from right to left, recording the breakpoints as the slopes of the visited edges (Line \ref{Line:slope edge upper hull}). Furthermore, the smallest minimizing indices associated with the breakpoints are computed in Line~\ref{Line:smallest indices}, from which the maximal minimizers are obtained immediately (Line \ref{Line:maximal minimizer possi}). 
Finally, the complexity is dominated by the upper convex hull computation (Line~\ref{Line:upper convex hull}), which is $O(n\log n)$ in general but is only $O(n)$ here since the dual points are already sorted by x-coordinate (the slopes of $g_j$ are strictly monotone)\citep[Theorem 1.1]{de2008computational}. Hence, we obtain the following.
\begin{proposition}
    Algorithm~\ref{algo:UE_pos} runs in $O(n)$ time.
\end{proposition}

\begin{algorithm}
\caption{Solve $\min_{A\subseteq \Omega}\Pi(A)-\alpha |A|~\forall \alpha$}
\label{algo:UE_pos}
\KwIn{possibility distribution $\pi$ s.t $1= \pi_1 \geq \pi_2 \geq \ldots \geq \pi_n >  \pi_{n+1}=0$.}

Form points $p^j = (j-n-1, -\pi_j)$, each has attribute $p^j.\mathrm{idx}=j$. \;
Compute upper hull vertices (from right to left) of the set $\{p^j: 1\leq j \leq n+1\}$ and store in array $V$ \label{Line:upper convex hull}\; 
$N$= length($V$); $\mathrm{B}=[~]$; $\mathrm{I} = [~]$\;
\For{$i= 1$ \KwTo $N-1$}{
    % $p = V[k]$ and  $q = V[k+1]$\;
    Add slope of edge joining $V[i]$ and $V[i+1]$ to $\mathrm{B}$ \label{Line:slope edge upper hull} \;
     Add $V[i+1].\mathrm{idx}$ to $\mathrm{I}$ \label{Line:smallest indices}\;
}
\For{$j \in \mathrm{I}$}
{
$S_j= \{j, \ldots,n\}$ \label{Line:maximal minimizer possi}\;
}
\Return{ $\mathrm{B}$ and $\{S_j\}$\;}

\end{algorithm}

\begin{example}\label{example:poss distribution}
    Consider the distribution $\pi= (1,0.6, 0.3)$. We have $g_1(\alpha)=1-3\alpha, g_2(\alpha)=0.6-2\alpha, g_3(\alpha)=0.3-\alpha$, and $g_4(\alpha)=0$. 
    The graph of $g=\min_{i}g_i$ is illustrated in Figure \ref{fig:min_gi} in which the intersections of active lines are marked in black. 
    At the first breakpoint $\alpha=0.3$, the smallest $j$ satisfying $g_j=g$ is $j=2$.
    At the second (and last) breakpoint $\alpha=0.4$, the smallest $j$ satisfying $g_j=g$ is $j=1$. From these breakpoints, we read the chain $\{S_j\}$ associated to $\min_{A\subseteq \Omega}\Pi(A)-\alpha |A|$ as: 
         $\forall \alpha \in [0.3, 0.4)$, the maximal minimizer is $\{2,3\}$;
         $\forall \alpha \in [0.4, +\infty)$, the maximal minimizer is $\{1, 2,3\}$.
         Of course,
     $\forall \alpha \in [0, 0.3)$, the maximal minimizer is $\emptyset$.

    Let $g^*_i$ be the dual points of lines $g_i$. Figure \ref{fig:upper_convex_hull} illustrates the convex hull of these points in which the upper convex hull is marked in red. As $g^*_3$ is not a vertex, the edges in this upper hull from right to left are $(g^*_4, g^*_2)$ and $(g^*_2, g^*_1)$, which correspond one to one with the intersections of active lines of $g$ (see Figure \ref{fig:min_gi}). We also see that the slopes of edges $(g^*_4, g^*_2)$ and  $(g^*_2, g^*_1)$ are 0.3 and 0.4 matching exactly the breakpoints of $g$ calculated before.

\end{example}

\begin{figure}[H]
    \centering
    \begin{subfigure}[b]{0.48\linewidth}
        \centering
    \includegraphics[width=1.1\linewidth]{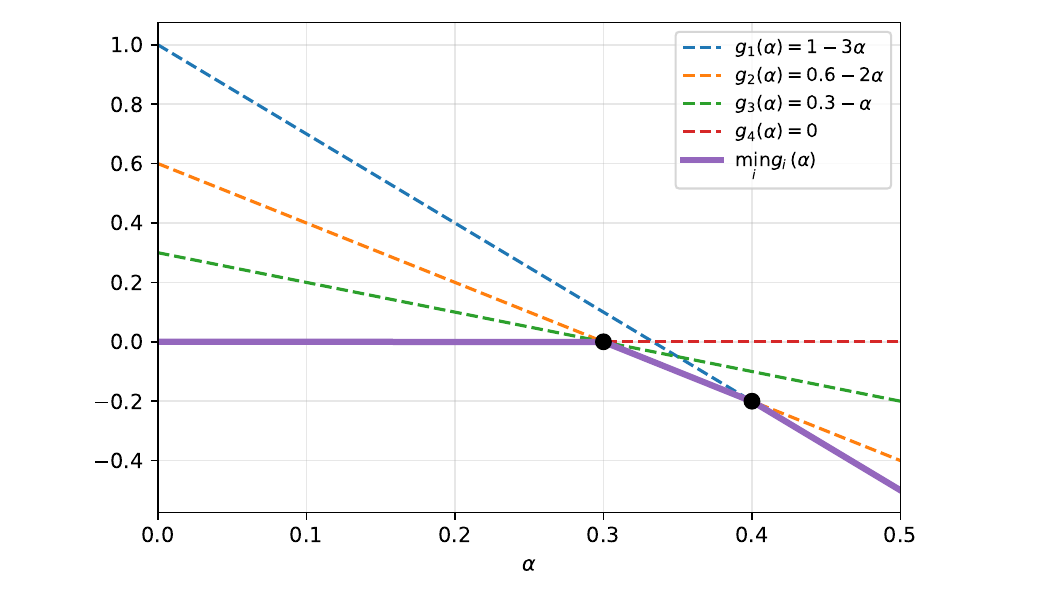}
    \caption{Lower envelope of lines $g_j$}
    \label{fig:min_gi}
    \end{subfigure}
    \hfill
    \begin{subfigure}[b]{0.48\linewidth}
         \centering
    \includegraphics[width=0.95\linewidth]{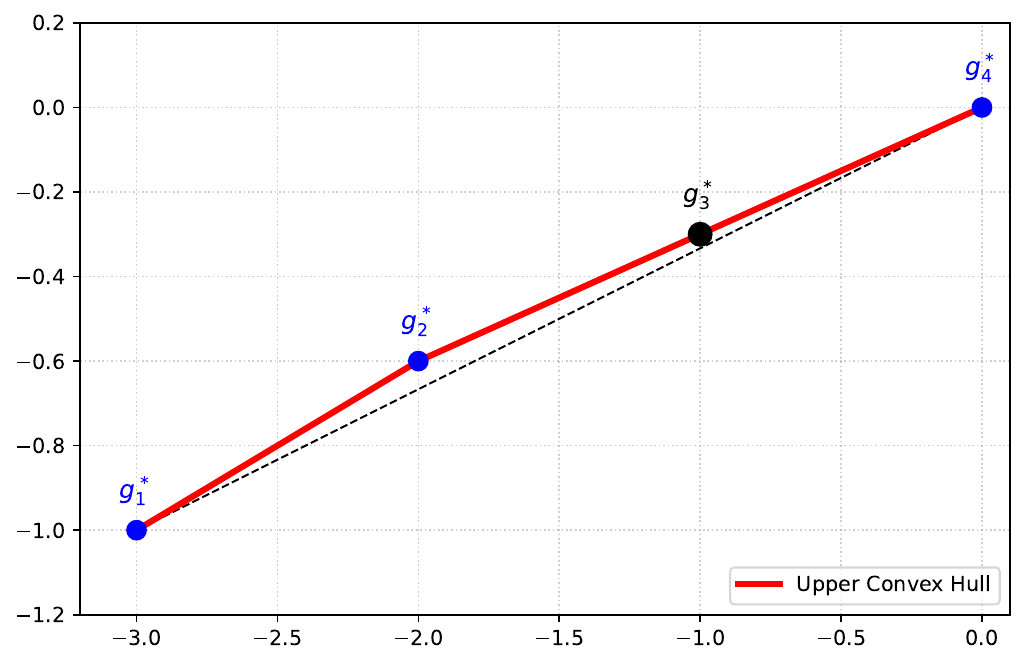}
    \caption{Upper convex hull of dual points $g^*_j$}
    \label{fig:upper_convex_hull}
    \end{subfigure}
    
    \caption{Illustration of Example \ref{example:poss distribution}}
    
    \label{fig:o}
\end{figure}

\subsection{Probability intervals}

In this case, computing the upper entropy becomes solving
\begin{equation}
\max_{p \in \mathcal{P}} -\sum_{i=1}^n p_i \log p_i,
\label{prob:UE_interval}
\end{equation}
where $\mathcal{P}= \{p \in \Delta_n: l_i \leq p_i\leq u_i~\forall i\}$.

\citet{abellan2003maximum} designed an algorithm for solving \eqref{prob:UE_interval} with $O(n^2)$ complexity. We now present an algorithm with better complexity. By using the KKT conditions \citep{boyd2004convex}, the optimal solution is characterized as follows. 
\begin{proposition}\label{prop:KKT_characterize_interval}
    $p$ is optimal of Problem \ref{prob:UE_interval} if and only if there exists $x$ such that $\sum_{i=1}^n\min\{\max\{x, l_i\}, u_i\}=1$ and $p_i = \min\{\max\{x, l_i\}, u_i\}~\forall i$.
\end{proposition}

Let $f_i(x) := \min\{\max\{x, l_i\}, u_i\} $ or explicitly  $$f_i(x)=
\begin{cases}
l_i, & x \leq l_i,\\
x, & l_i< x < u_i\\
u_i, & x \geq u_i.
\end{cases}$$
From Proposition \ref{prop:KKT_characterize_interval}, whose technical proof is in the Appendix~\ref{app:prob_int}, it suffices  to solve $f(x)=1$ where $f(x):= \sum_{i=1}^n f_i(x)$. Because each $f_i$ is non-decreasing and piecewise affine, the same holds for $f$. 
Let $a := \min_{i}l_i$ and $b := \max_{i}u_i$, then $f(a) = \sum_{i=1}^n l_i$ and $f(b) = \sum_{i=1}^n u_i$. As $f$ is continuous and $\sum_{i=1}^n l_i \leq 1 \leq \sum_{i=1}^n u_i$, from the intermediate value theorem there is a solution to $f(x)=1$ in the interval $[a, b]$.  Figure \ref{fig:plot_f} illustrates a plot of $f$. 
% Note also that $f$ has at most $2n$ breakpoints, which are (usually strictly) contained in $\cup_{i=1}^n\{l_i, u_i\}$.
\begin{figure}[H]
\centering
\includegraphics[width=0.5\linewidth]{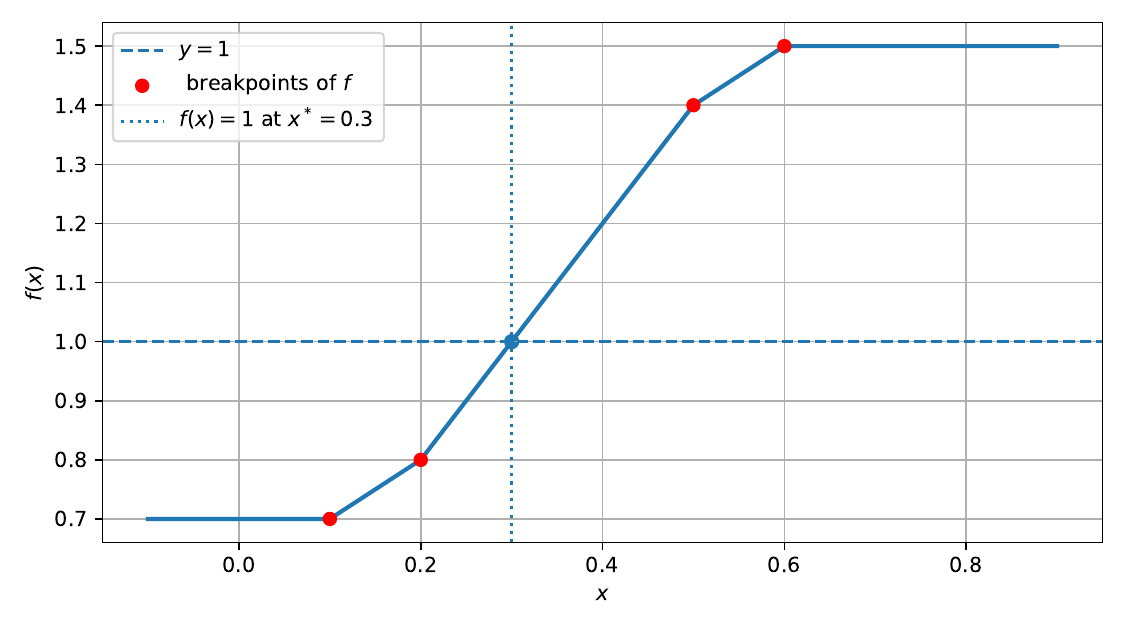}
\caption{Plot of $f(x)$ with $[l_1, u_1]=[0.1, 0.4], [l_2, u_2]=[0.4, 0.5]$ and $[l_3, u_3]=[0.2, 0.6]$.}
% The breakpoints of $f$ are colored in red.
\label{fig:plot_f}
\end{figure}

Because $f$ is non-decreasing, we can solve $f(x)=1$ by binary search: at each iteration, set $x=\frac{a+b}{2}$; if $f(x)<1$, continue on $[x,b]$, otherwise continue on $[a,x]$, and repeat. This procedure is already quite efficient. In fact, suppose we aim to solve $f(x)=1$ with accuracy $\epsilon$, i.e., $|f(x)-f(x^*)| \leq \epsilon$ where $x^*$ is a solution. Because $|f(x)-f(x^*)| \leq n |x-x^*|$, we stop the binary search once the interval width is less than $\frac{\epsilon}{n}$, which requires at most $\lceil\log_2\frac{n}{\epsilon}\rceil$ iterations (as $[a,b]\subseteq [0,1]$). For example, if $n=10^7$ and $\epsilon=10^{-10}$, this amounts to about $56$ iterations, each needing only an evaluation of $f$.

We can speed up the computation even further by combining binary search with the celebrated Newton method. Start with an initial guess $x_0$, until a solution is found, Newton’s method iteratively updates $x_k$ for $k=1,2,\ldots$ as
\begin{equation}\label{eq:Newton}
    x_{k} = x_{k-1} - \frac{f(x_{k-1})-1}{f'(x_{k-1})}.
\end{equation}
Note that $f$ is not differentiable at its breakpoints. However, this is not problematic since at a given $x$, if $f(x)<1$ (resp. $f(x) >1$), the solution lies to the right (resp. the left) of $x$, and we replace the derivative in (\ref{eq:Newton}) with the right derivative $f'_+(x)$ (resp. left derivative $f'_{-}(x)$). It is easy to see that $\forall i$, $(f'_{i})_+(x) = 1$ if $ l_i \leq x < u_i$ and $(f'_{i})_+(x) = 0$ otherwise. Summing these right derivatives, we obtain
$f'_+(x) = |\{i:  l_i \leq x < u_i \}|$. Similarly, $f'_-(x) = |\{i:  l_i  <x \leq u_i \}|$. Finally, the Newton step in (\ref{eq:Newton}) may step out of the current interval which contains the solution or be undefined due to a zero derivative. In these cases, we fall back to a binary-search step \citep[Chapter 9]{press1992numerical}. Wrapping up, the approach is described in Algorithm~\ref{algo:Newton-binary}. At worst, it may perform binary search only, and thus we obtain:
\begin{proposition}
     Algorithm \ref{algo:Newton-binary} runs in $O(n\log \frac{n}{\epsilon})$ time.
\end{proposition}
% \begin{proof}
%     As in the worst case, it can perform binary search only.
% \end{proof}
We note that in practice, Algorithm \ref{algo:Newton-binary} is often considerably faster due to the Newton steps.

\begin{algorithm}[t]
\caption{Hybrid Newton--Binary Search}
\label{algo:Newton-binary}
\KwIn{ $n$-intervals $[l_i, u_i] \subseteq [0,1]$; tolerance $\epsilon$}
$a= \min_{i}l_i,~b= \max_{i}u_i$, $x= \frac{a+b}{2}$\;

\While{$|f(x)-1|>\epsilon$}{
\eIf{$f(x)-1<0$}{$a=x$ and $d= |\{i:  l_i \leq x < u_i \}|$}
{$b=x$ and $d= |\{i:  l_i  <x \leq u_i \}| $}
\eIf{$d=0$}{$x= \frac{a+b}{2}$\;
}{
$x^\prime= x - \frac{f(x)-1}{d}$\;
    \leIf{$x^\prime \in (a,b)$}{
        $x= x^\prime$
    }{
        $x= \frac{a+b}{2}$
    }
}
}
\Return{$p_i=\min\{\max\{x, l_i\}, u_i\}~\forall i$}

\end{algorithm}

\section{Computing upper entropy via convex optimization}\label{sec:FW_algo}
The bottleneck of Algorithm \ref{algo:decomposition} lies in solving $O(n)$ SFM since  algorithms for solving a general SFM are slow due to their high-order polynomial complexity. Hence, for general 2-monotonicity, it is worthwhile to explore efficient approximate ways to compute upper entropy. 

We now adopt the well-known Frank--Wolfe (FW) algorithm \citep{Bachsubmodular13} to solve Problem \eqref{prob:upper_entropy}. Denote the entropy function by $H(p):= -\sum_{i=1}^np_i \log p_i$. In this method, we start with a  ${p}^0 \in \mathcal{P}(\mu) $. 
At each iteration $k= 0,1,2, \ldots$, we solve the so-called linear oracle
\begin{equation}\label{eq:FW_oracle}
    \max_{p \in \mathcal{P}(\mu) } \nabla H(p^k)^T (p-p^k),
\end{equation}
for which the maximum is attained at $v^k$. Since $\mathcal{P}(\mu)$ is convex, the segment between $p^k$ and $v^k$ lies in $\mathcal{P}(\mu)$, and the next iterate is taken on this segment as the point with the best objective, i.e., 
$p^{k+1} = (1-\gamma_k)p^k+\gamma_kv^k$ where 
\begin{equation}\label{eq:line_search}
    \gamma_k \in \argmax_{\gamma \in  [0,1]} \, H((1 - \gamma)p^k + \gamma v^k).
\end{equation}
If $p^*$ is the optimal probability then the concavity of $H$ ensures that $H(p^k) \leq H(p^*) \leq  H(p^k) + \nabla H(p^k)^T (v^k-p^k)$. Hence, the maximum value of (\ref{eq:FW_oracle}) (also called the \emph{duality gap}) serves as a stopping certificate: we terminate if it is smaller than a threshold.
As $H$ is not differentiable at $p$ when some $p_i=0$, we apply the algorithm by starting with an initial point $p^0 >0$ and restricting the line search (\ref{eq:line_search}) to $\gamma\in[0,1)$. Thus, the gradient is always defined. 
% This way, every iterate remains in the interior of the probability simplex. 
\begin{remark}\label{remark:initial_p0_FW}
    A necessary and sufficient condition for the existence of such $p^0>0$ is $\bar{\mu}(\{i\}) >0 ~\forall i$. Since $\bar{\mu}(\{i\})=0$ means $i$ is impossible (as $\bar{\mu}$ provides upper bounds for event probabilities), this condition makes all outcomes $i$ possible. Moreover, $p^0$ can be found in $O(n^2\mathrm{EO})$ (see Appendix \ref{proof: remark creating p0}). 
\end{remark}
The key step in the algorithm is solving \eqref{eq:FW_oracle}. Fortunately, $\mathcal{P}(\mu)$ coincides with the base polyhedron of $\bar{\mu}$ (see Remark \ref{remark:credalset=base poly}) and Problem \eqref{eq:FW_oracle} is linear in $p$. Hence, it is efficiently solvable due to the following result.
\begin{proposition}[\citep{fujishige2005submodular}]
    Let $w \in \mathbb R^n$ and $\phi$ be a permutation such that $w_{\phi(1)}\geq \ldots\geq w_{\phi(n)}$. Define $S_i := \{\phi(1), \ldots, \phi(i)\}~\forall i \in \{1, \ldots, n\}$ and  $S_0 = \emptyset$. If $\nu$ is a submodular function on $\Omega$, then a maximizer of $\max_{x\in B(\nu)}w^Tx$ is given as $x_i = \nu(S_i)-\nu(S_{i-1})~\forall i$.
\end{proposition}

\section{Experiments}\label{sec:expe}
 Our experiments are conducted on a laptop with 16GB RAM and Intel Core i9 where we code in Python with NumPy package for efficient array operations \citep{harris2020array}. Our code is available in the supplementary materials.
\subsection{2-monotonicity}
To construct a 2-monotone $\mu$, we do as follows \citep{bednarski1981solutions}: let $f:[0,1] \to [0,1]$ be any convex function where $f(0)=0$ and $p^*$ be any probability on $\Omega$; define $ \mu(\Omega)=1,~\mu(A) = f(p^*(A)) ~\forall A\neq \Omega.$ We have chosen $p^*$ as a vector sampled from a Dirichlet distribution where all concentration parameters are 1. Our choices of $f$ and sizes of $\Omega$ are listed in Table \ref{tab:FW_resolution}. 

We implement the FW algorithm for this type of 2-monotonicity. We first generate the initial point $p^0 >0$ in the credal set which takes $O(n^2\mathrm{EO})$ (see Remark \ref{remark:initial_p0_FW}). The resolution is then carried out starting from $p^0$. At any iteration, we obtain an approximate entropy ${H}_{a}$ and a $\mathrm{gap}$ for which the true entropy $H^*\in [H_{a}, H_{a}+ \mathrm{gap}]$. Hence, the relative error $\frac{H^*-H_{a}}{H^*}$ is bounded by $\frac{\mathrm{gap}}{H_{a}+ \mathrm{gap} }$ and we stop once this bound is less than 0.1\%. Table \ref{tab:FW_resolution} reports the average running times (in seconds) for initial-point generation and the subsequent resolution, obtained by repeating each $(f, |\Omega|)$ experiment with five random seeds used to construct $\mu$.
% $\frac{H^*-H_{a}}{H^*}$$=1 - \frac{H_{a}}{H^*}\leq \frac{\mathrm{gap}}{H_{a}+ \mathrm{gap} }$. 
From Table \ref{tab:FW_resolution}, we observe that the approach is quite efficient, yet most of the computation time is spent generating the initial point $p^0 >0$ rather than the resolution. This is because the former has quadratic complexity, making it the bottleneck as $|\Omega|$ grows. 

Therefore, when the space is large, we may adopt the workaround of optimizing the smooth surrogate $\tilde{H}(p):= -\sum_{i=1}^n p_i \log(p_i+\epsilon)$ for a fixed, sufficiently small $\epsilon>0$. The gap between the optimal values of $H$ and $\tilde{H}$ is bounded by $n\epsilon$ \cite[Lemma A.1]{hazan2019provably}. With this surrogate, we can start from any feasible point in the credal set, which is generated in $O(n)$ time using a single permutation on $\Omega$ (see Appendix \ref{proof: remark creating p0}). The results in Table \ref{tab:FW_resolution_surrogate} show that this approach is suitable for large spaces, e.g., $|\Omega|=5\times 10^5$. 

In general, this approximation approach is likely to outperform in efficiency the exact Algorithm~\ref{algo:decomposition}, mostly because solving the SFM problems will require polynomial methods whose exponent will be too high. We would therefore recommend it for very high values of $|\Omega|$, for which Algorithm~\ref{algo:decomposition} will struggle. 

\begin{table}[H]
\centering
\caption{Initial point construction and resolution times (s).}
\label{tab:FW_resolution}
\begin{tabular}{l >{\small}r >{\small}r >{\small}r}
\toprule
 $f$
& \multicolumn{1}{c}{$|\Omega|$}
& \multicolumn{1}{c}{\shortstack{Create $p^0$}}
& \multicolumn{1}{c}{\shortstack{Resolution}} \\
\midrule
\multirow{3}{*}{$x^2$} & 5000  & $0.284 \pm 0.006 $ & $0.013 \pm 0.001$ \\
                      & 10000  & $1.035 \pm 0.018$ & $0.024 \pm 0.001 $ \\
                      & 20000 & $3.961 \pm 0.072$ & $ 0.206\pm 0.027$ \\
\midrule
\multirow{3}{*}{$1 - \sqrt{1-x}$} & 5000  & $0.307 \pm 0.003$ & $0.009 \pm 0.001$ \\
                      & 10000  & $1.149 \pm 0.016$ & $0.017 \pm 0.001$ \\
                      & 20000 & $4.333 \pm 0.014$ & $0.147 \pm  0.014$ \\
\midrule
\multirow{3}{*}{$\frac{\mathrm{e}^{2x}-1}{\mathrm{e}^2-1}$} & 5000  & $0.397 \pm 0.003$ & $0.008 \pm 0.001$ \\
                      & 10000  & $1.481 \pm 0.010$ & $0.015 \pm 0.001$ \\
                      & 20000 & $5.713 \pm 0.019$ & $ 0.139 \pm 0.010$ \\
\bottomrule
\end{tabular}
\end{table}

\begin{table}[t]
\centering
\caption{Initial point construction and resolution times (s) for $|\Omega|=5\times 10^5$ with the surrogate.}
\label{tab:FW_resolution_surrogate}
\begin{tabular}{l >{\small}r >{\small}r}
\toprule
 $f$
% & \multicolumn{1}{c}{$|\Omega|$}
& \multicolumn{1}{c}{\shortstack{Create $p^0$}}
& \multicolumn{1}{c}{\shortstack{Resolution}} \\
\midrule
{$x^2$}   & $0.010 \pm 0.001 $ & $6.327 \pm 0.291$ \\
\midrule
{$1 - \sqrt{1-x}$}   & $0.012 \pm 0.001$ & $4.326 \pm 0.088$ \\
                     
\midrule
{$\frac{\mathrm{e}^{2x}-1}{\mathrm{e}^2-1}$}   & $ 0.019 \pm 0.001$ & $3.617 \pm 0.074$ \\
                     
\bottomrule
\end{tabular}
\end{table}

\subsection{Probability intervals}
To ensure a nonempty credal set, we construct $n$ random intervals $[l_i, u_i]$ as follows. First, we draw all $u_i$ independently and uniformly from $(0, 1]$ and compute $\sum_{i=1}^n u_i$. If this sum is less than 1, we discard the draw and resample. Second, each $l_i$ is drawn uniformly from $[0, u_i)$. Finally, we rescale  $l_i$ so that $\sum_{i=1}^n l_i = 1 - \mathrm{margin}$ for a fixed $0 <\mathrm{margin} <1$. We compare  Algorithm \ref{algo:Newton-binary} against the one in \citep{abellan2003maximum} and report in Table \ref{tab: newton vs AM} the average running times (in seconds) over five random seeds for each $(\text{margin}, |\Omega|)$ pair.

As mentioned earlier, Algorithm \ref{algo:Newton-binary} has better complexity than Abellán \& Moral's algorithm ($O(n\log \frac{n}{\epsilon})$ vs. $O(n^2)$). From Table \ref{tab: newton vs AM}, we also observe a significant improvement in the experiment. Moreover, the results suggest that their algorithm is influenced not only by $|\Omega|$ but also by the size of the feasible set: for a fixed $|\Omega|$, increasing $\mathrm{margin}$ enlarges the credal set and their algorithm struggles more. Our algorithm, in contrast, appears to depend only on $|\Omega|$.

\begin{table}[H]
\centering
\caption{Runtimes (s) of Algorithm \ref{algo:Newton-binary} vs. Abellán \& Moral}
\label{tab: newton vs AM}
\begin{tabular}{l >{\small}r >{\small}r >{\small}r}
\toprule
 $\mathrm{margin}$
& \multicolumn{1}{c}{$|\Omega|$}
& \multicolumn{1}{c}{\shortstack{Algorithm \ref{algo:Newton-binary}}}
& \multicolumn{1}{c}{\shortstack{Abellán \& Moral}} \\
\midrule
\multirow{2}{*}{$0.1$} & $10^7$  & $0.305\pm 0.002 $ & $12.906\pm 0.036$ \\
                      & $10^8$ & $3.133\pm 0.042 $ & $15.699\pm 0.144 $ \\

\midrule
\multirow{2}{*}{$0.2$} & $10^7$  & $0.265\pm 0.003$ & $19.375\pm 0.047$ \\
                      & $10^8$ & $2.729\pm 0.037$ & $19.589\pm 0.245$ \\
 \midrule
\multirow{2}{*}{$0.3$} & $10^7$  & $0.281\pm 0.015$ & $25.054\pm 0.060$ \\
                      & $10^8$ & $2.843\pm 0.147 $ & $27.069\pm 0.179$ \\

\bottomrule
\end{tabular}
\end{table}

\section{Conclusion}
In this paper, we revisited the problem of computing upper entropy for 2-monotone lower probabilities, from theoretical and practical viewpoints. Upper entropy indeed plays an important role when it comes to credal uncertainty quantification, hence the interest to fully explore its computational aspects. 

We analyzed  Abellán \& Moral's algorithm, well-known in imprecise probability but whose exact complexity was previously unknown. In contrast with previous claims saying that it was difficult, we showed that it was strongly polynomial, and proposed an improved algorithm based on submodular optimization. For important special cases, namely belief functions, probability intervals and possibility distributions, we also developed algorithms with notable improvements over existing approaches in the literature. Along with those improved methods, our results give a rather comprehensive picture of how to compute upper entropies for practical credal sets routinely used in applications. 

Also, as the algorithm for computing upper entropy in general cases has high-degree polynomial complexity, we proposed an alternative approach using the classic Frank-Wolfe method. Initial experiments indicates that this alternative is practical and achieves good precision quickly, thus offering a highly efficient alternative to exact methods. In the future, we aim to perform much more comprehensive experiments. In particular, for belief functions, it is interesting to compare the exact resolution approach using maximum-flow algorithms with the approximate one via Frank-Wolfe.

A more general question to which this paper contributes is to know to which extent classical results issued from submodular optimization can help to solve commonly encountered computational issues for credal sets, and more specifically those induced by 2-monotone lower probabilities? One could think of other robust versions of classical uncertainty quantification tools, such as KL divergences, Wasserstein metrics. Another interesting topic to explore under this algorithmic and computational aspect is optimal transport~\citep{lorenzini2025choquet,caprio2025optimal}.

\bibliography{bib}

\newpage

\onecolumn

% \title{Upper entropy for 2-monotone lower probabilities\\(Supplementary Material)}
% \maketitle
\title{Upper entropy for 2-monotone lower probabilities\\(Appendix)}
\maketitle

\appendix

% \section{Additional Proofs}
\section{Proof of Lemma \ref{Lemma:1}}\label{appendix: proof of lemma 1}

Before proving this, we remark that if $\mu$ is supermodular, the set function $\mu_{\lambda}(A):= \mu(A) - \lambda |A|$ is also supermodular as $|A| + |B| = |A\cap B| +|A\cup B|~\forall A,B$.
\paragraph{Part 1:} Finding $A\in \argmax_{\emptyset \neq B \subseteq \Omega} \frac{\mu(B)}{|B|}$ by solving $O(n)$ SFM.

For this part, we adapt the idea described in \citep[Sec. 4.1]{fleischer2003push}.

\begin{proof}
Define a function $g(\lambda) := \max_{\emptyset \neq B \subseteq \Omega} \mu(B)-\lambda|B|$. We claim that 
\begin{equation}\label{eq:root DSS}
  \lambda^*:=\max_{\emptyset \neq B \subseteq \Omega} \frac{\mu(B)}{|B|}   \Leftrightarrow g(\lambda^*)=0.
\end{equation} 
Indeed, if $\max_{B\neq \emptyset} \frac{\mu(B)}{|B|} > \lambda$ then $\max_{B\neq \emptyset} \frac{\mu(B)-\lambda|B|}{|B|}>0$, and thus it follows that $g(\lambda) >0$ because $|B| >0$. Hence, we need to solve $g(\lambda)=0$, for which the Dinkelbach’s method \cite{dinkelbach1967nonlinear} will be employed. Note that $\lambda^*\in$ $(0,1]$. Dinkelbach’s method then constructs the sequence $\{\lambda_t\}$ starting from $\lambda_0=0$, evaluates $g(\lambda_t)$, and updates $\lambda_t$ until it hits a solution of $g$ as
\begin{enumerate}
    \item Find $B_t\in \operatorname{argmax}_{B \neq \emptyset}\mu(B)-\lambda_t|B|$. Evaluate $g(\lambda_t)=\mu(B_t)-\lambda_t|B_t|$.
    \item If $g(\lambda_t)=0$, then $\lambda_t$ is a solution and $B_t$ is optimal. 
    \item \label{it: Dinkelback_update} Otherwise, update $\lambda_{t+1}:=\lambda_t +\frac{g(\lambda_t)}{|B_t|}= \frac{\mu(B_t)}{|B_t|}  $.
\end{enumerate}
Define $M_k:=\max_{|B|=k}\mu(B)$ for $k = 1, \ldots,n$. At each iteration, because $|B_t|$ is a maximizer, if $|B_t|=k$, it follows that $\mu(B_t) = M_k$, and thus $\lambda_{t+1}=\frac{M_k}{k}$. Consequentiality, the sequence $\{\lambda_t\}$ can only take values in the finite set $\{\frac{M_1}{1}, \ldots, \frac{M_n}{n}\}$. Since $\{\lambda_t\}$ strictly increases until optimality with $g(\lambda_t)=0$, the loop terminates in at most $n$ iterations. Finally, at each iteration we solve $\max_{B}\mu(B)-\lambda_t|B|$ which is an SFM as $\mu$ is supermodular. In total, we need to solve $O(n)$ SFM. 
\end{proof}

\paragraph{Part 2:} Finding the maximizer $A$ of maximum size by solving $O(n)$ SFM.
\begin{proof}
    Because the set of maximizers of a supermodular function is closed under union, after finding $\lambda^*$ (see \eqref{eq:root DSS}) via Dinkelbach’s method, we simply take the union of all maximizers of $\mu_{\lambda^*}$ where $\mu_{\lambda^*}(B):= \mu(B) - \lambda^*|B|~\forall B$. Observe that $i\in \Omega$ belongs to a maximizer of $\mu_{\lambda^*}$ if and only if
    \begin{equation}\label{eq:i_in_a_maximizer}
        \max_{B\subseteq \Omega\setminus\{i\}}\mu(B \cup \{i\})- \lambda^*|B\cup \{i\}|=0.
    \end{equation}
    After solving \eqref{eq:i_in_a_maximizer} $n$ times for each $i\in \Omega$, the desired set $A$ consists of all $i$ such that \eqref{eq:i_in_a_maximizer} holds.
\end{proof}

\section{Proof of Proposition \ref{prop:KKT_characterize_interval}}\label{app:prob_int}
WLOG, assume that $u_i > 0 ~\forall i$ (otherwise $p_i = 0$ and we discard $p_i$),  $\sum_{i=1}^n l_i < 1$ (otherwise $p_i = l_i ~\forall i$), and $\sum_{i=1}^n u_i > 1$ (otherwise $p_i = u_i ~\forall i$). Under these assumptions, at the optimal $p$, we have $p_i > 0~\forall i$.

\begin{proof}
     By the KKT conditions, $p$ is optimal of Problem \eqref{prob:UE_interval} if and only if the following system (with variables $p_i, \alpha_i, \beta_i$ and $\lambda$) is feasible:
\begin{align}
    &l_i \leq p_i \leq u_i \text{ and } \sum_{i=1}^n p_i= 1 \label{constr:primal}\\
    &\alpha_i, \beta_i \geq  0 ~\forall i \label{constr:dual}\\
    &\alpha_i(-p_i + l_i) = 0 \text { and } \beta_i(p_i - u_i) = 0 ~\forall i \label{constr:stackness} \\
    &\log p_i + 1 -\alpha_i + \beta_i + \lambda = 0 ~\forall i \label{constr:vanish_grad}
\end{align}

Let $f(x):=\sum_{i=1}^n\min\{\max\{x, l_i\}, u_i\}$ and $b = \max_{i}u_i$. Hence, $f(0)=\sum_{i=1}^nl_i <1$  and  $f(b)=\sum_{i=1}^nu_i > 1$ (by our assumptions). As $f$ is continuous, by the intermediate value theorem, there exists an $x\in (0, b)$ such that $f(x)=1$. 
We define $p_i = \min\{\max\{x, l_i\}, u_i\} ~\forall i$. We claim that such $p_i$ is feasible to (\ref{constr:primal}-\ref{constr:vanish_grad}) and thus is optimal.  
By the definition of $p_i$, \eqref{constr:primal} automatically holds. We set  $\lambda = -1-\log x.$ For each $i$, if $ l_i < u_i$, we consider three cases: 
\begin{enumerate}
    \item If $l_i <x <u_i$ then $p_i =x$, and we set $\alpha_i = \beta_i=0$.
    \item If $x \leq l_i$ then $p_i =l_i$, and we set $\alpha_i = \log\frac{l_i}{x}$ and $\beta_i=0$.
    \item If $x \geq u_i$ then $p_i =u_i$, and we set $\alpha_i = 0$ and $\beta_i=\log\frac{x}{u_i}$. 
\end{enumerate}
Finally, if $l_i = u_i$, then $p_i = l_i$, and we can set any $\alpha_i, \beta_i$ such that $-\alpha_i + \beta_i = \log \frac{x}{l_i}$.
\end{proof}

\section{Proof of Remark \ref{remark:initial_p0_FW}}\label{proof: remark creating p0}
We make use of the following well-known result about the structure of vertices of $\mathcal{P}(\mu)$ \citep{fujishige2005submodular}. Let $\phi$ be a permutation on $\Omega$, the probability $p$ where $p_{\phi(i)}:= \bar{\mu}(S^\phi_i)-\bar{\mu}(S^\phi_{i-1})$ with $S^\phi_0:= \emptyset$, $S^\phi_i:=\{{\phi(1)}, \ldots, {\phi(i)}\}~\forall i$, is a vertex of $\mathcal{P}(\mu)$.

\begin{proof}
   If $\exists p \in \mathcal{P}(\mu)$ such that $p_i >0~\forall i$ then $\bar{\mu}(\{i\})> p_i > 0$. Conversely, assume that  $\bar{\mu}(\{i\})>0~\forall i$. Consider $n$ permutations $\phi^1, \ldots, \phi^n$ on $\Omega$ where $\phi^i(1)=i~\forall i$. Because of the above-mentioned result, each $\phi^i$ induces a vertex of $\mathcal{P}(\mu)$ for which its $i$th component is positive. Taking the average of these $n$ vertices, we obtain a probability $p \in \mathcal{P}(\mu)$ such that $p_i >0 ~\forall i$. Moreover, this $p$ is found in $O(n^2\mathrm{EO})$.
\end{proof}

\end{document}